\documentclass[11pt, a4paper]{article}
\usepackage[a4paper]{geometry}
\usepackage{amsmath,amsthm,amssymb,mathtools}
\renewenvironment{proof}
{\begin{trivlist}\item\textbf{Proof.}}
{\hspace*{\fill}$\Box$\end{trivlist}}

\newenvironment{proofof}[1]
{\begin{trivlist}\item\textbf{Proof of #1.}}
{\hspace*{\fill}$\Box$\end{trivlist}}
\renewcommand{\labelenumi}{(\alph{enumi})}
\renewcommand\theenumi\labelenumi
\newcommand{\filtt}{\mathcal{F}_t}
\newtheorem{lemma}{Lemma}
\newtheorem{theorem}{Theorem}
\newtheorem{definition}{Definition}

\usepackage{url}
\usepackage[numbers,sort&compress]{natbib}
\usepackage{xspace,xcolor}
\usepackage{dsfont}
\usepackage[ruled,vlined,linesnumbered,algo2e]{algorithm2e}
\SetKwFor{Function}{function}{is}{end function}
\newcommand{\assign}{\leftarrow}
\SetKw{Input}{input}
\SetKw{Output}{output}
\SetKw{Initialization}{initialization}
\newcommand{\R}{{\mathds{R}}}
\newcommand{\filttminone}{\mathcal{F}_{t-1}}
\newcommand{\filtzero}{\mathcal{F}_0}
\newcommand{\ignore}[1]{}
\newcommand{\E}[1]{\mathrm{E}(#1)}
\newcommand{\Ebig}[1]{\mathord{\mathrm{E}}\mathord{\left(#1\right)}}

\newcommand{\xmin}{x_{\mathrm{min}}}
\newcommand{\OneMax}{\textsc{OneMax}\xspace}
\newcommand{\LeadingOnes}{\textsc{LeadingOnes}\xspace}
\newcommand{\om}{\OneMax}\newcommand{\lo}{\textsc{LeadingOnes}\xspace}\newcommand{\wrt}{w.\,r.\,t.\xspace}
\newcommand{\ie}{i.\,e.\xspace}
\newcommand{\eg}{e.\,g.\xspace}
\newcommand{\card}[1]{|#1|}
\renewcommand{\epsilon}{\varepsilon}
\DeclareMathOperator{\Prob}{Pr}
\newcommand{\ea}{(1+1)~EA\xspace}
\newcommand{\oneoneea}{\ea}
\begin{document}

\author{Timo Kötzing\\
Hasso Plattner Institute\\
Potsdam\\
Germany
\and
Carsten Witt\\
Technical University of Denmark\\Kgs. Lyngby\\Denmark}

\title{Improved Fixed-Budget Results via Drift Analysis}

\maketitle 

\begin{abstract}
Fixed-budget theory is concerned with computing or bounding the fitness value 
achievable by randomized search heuristics within a given budget of fitness 
function evaluations. Despite recent progress in fixed-budget theory, there is a 
lack of general tools to derive such results. 
We transfer drift theory, the key tool to derive expected optimization times, to the fixed-budged perspective. 
A first and easy-to-use statement concerned with iterating drift in so-called greed-admitting scenarios 
immediately translates into bounds on 
the expected function value. 
Afterwards, we consider a more general tool based on the well-known variable drift theorem. 
Applications 
of this technique to the \textsc{LeadingOnes} benchmark function yield statements 
that are more precise than the previous state of the art. 
\end{abstract}

\section{Introduction}

Randomized search heuristics are a class of optimization algorithms which use probabilistic choices with the aim of maximizing or minimizing a given objective function. Typical examples of such algorithms use inspiration from nature in order to determine the method of search, most prominently evolutionary algorithms, which use the concepts of \emph{mutation} (slightly altering a solution) and \emph{selection} (giving preference to solutions with better objective value).

The theory of randomized search heuristics aims at understanding such heuristics by explaining their optimization behavior. Recent results are typically phrased as run time results, for example by giving upper (and lower) bounds on the expected time until a solution of a certain quality (typically the best possible quality) is found. This is called the (expected) \emph{optimization time}. A different approach, called \emph{fixed-budget analysis}, bounds the quality of the current solution of the heuristic after a given amount of time. In order to ease the analysis and by convention, in this theoretical framework \emph{time} is approximated as the number of evaluations of the objective function (called fitness evaluations).

In this paper we are concerned with the approach of giving a fixed-budget analysis. This approach was introduced to 
the analysis of randomized search heuristics by Jansen and Zarges~\cite{JansenZargesGECCO12}, who derived fixed-budget results 
for the classical example functions \om and \lo by bounding the expected progress in each iteration. A different perspective 
was proposed by Doerr, Jansen, Witt and Zarges \cite{DJWZGECCO13}, who showed that fixed-budget statements can be derived from bounds on optimization times 
if these exhibit strong concentration. Lengler and Spooner~\cite{LenglerSpoonerFOGA15} proposed a variant of multiplicative drift for 
fixed-budget results and the use of differential equations in the context of \om and general linear functions. 
Nallaperuma, Neumann and Sudholt \cite{NallaperumaNSECJ17} applied fixed-budget theory to the analysis of evolutionary algorithms on the 
traveling salesman problem 
and Jansen and Zarges~\cite{JansenZargesTEC2014} to artificial immune systems. The quality gains of optimal black-box algorithms 
on \om in a fixed-budget perspective were analyzed by Doerr, Doerr and Yang \cite{DoerrDYTCS20}. In a recent technical report,
He, Jansen and Zarges \cite{HeJansenZargesArxiv20} consider so-called unlimited budgets to estimate fitness values in particular for points of time 
larger than the expected optimization time. A recent survey by Jansen~\cite{JansenFixedBudgetSurvery} 
summarizes the state of the art in the area of fixed-budget analysis.

There are general methods easing the analysis of randomized search heuristics. Most importantly, in order to derive bounds on the optimization time, we can make use of drift theory.  Drift theory is a general term for a collection of theorems that consider random processes and bound the expected time it takes the process to reach a certain value –- the first-hitting time. The beauty and appeal of these theorems lie in them usually having few restrictions but yielding strong results. Intuitively speaking, in order to use a drift theorem, one only needs to estimate the expected change of a random process –- the drift –- at any given point in time. Hence, a drift theorem turns expected local changes of a process into expected first-hitting times. In other words, local information of the process is transformed into global information. See~\cite{Lengler2020} for an extensive discussion of drift theory.

In contrast to the numerous drift theorems available for bounding the optimization time, there is no corresponding theorem for making a fixed-budget analysis apart from one for the multiplicative case given in~\cite{LenglerSpoonerFOGA15}. With this paper we aim to provide several such drift theorems, applicable in different settings and with a different angle of conclusions. In each our main goal is to provide an \emph{upper bound} on the distance to the optimum after $t$ iterations, for $t$ less than the expected optimization time. Upper bounds alone do not allow for a fair comparison of algorithms, since a bad upper bound does not exclude the possibility of a good performance of an algorithm; for this, we require lower bounds. However, one of our techniques also allows us to derive lower bounds. Furthermore, when upper and lower bounds are close together we can conclude that the derived bounds are correspondingly tight, highlighting the quality of our methods.

We start, in Section~\ref{sec:directDrift}, by giving a theorem which iteratively applies local drift estimates to derive a global drift estimate after $t$ iterations. Crucial for this theorem is that the drift condition is \emph{unlimited time}, by which we mean that the drift condition has to hold for all times $t$, not just (which is the typical case in the literature for drift theorems) those before the optimum is hit. This theorem is applicable in the case where there is no optimum (and optimization progresses indefinitely) and in the case that, in the optimum, the drift is $0$. In order to bypass these limitations we also give a variant in Section~\ref{sec:directDrift} which allows for \emph{limited time} drift, where the drift condition only needs to hold before the optimum is hit; however, in this case we pick up an additional error term in the result, derived from the possibility of hitting the optimum within the allowed time budget of $t$. Thus, in order to apply this theorem, one will typically need concentrations bounds for the time to hit the optimum.

For both these theorems, the drift function (bounding the drift) has to be convex and \emph{greed-admitting}, which intuitively says that being closer to the goal is always better in terms of the expected state after an additional iteration, while search points closer to the goal are required to have weaker drift. These conditions are fulfilled in many sample applications; as examples we give analyses of the \ea on \lo and \om. Note that these analyses seem to be rather tight, but we do not offer any lower bounds, since our techniques crucially only apply in one direction (owing to an application of Jensen's Inequality to convex drift functions).

In Section~\ref{sec:variableFixedBudget} we use a potential-based approach and give a variable drift theorem for fixed-budget analysis. As a special case, where
the drift function is constant, we give an additive drift theorem for fixed-budget analysis and derive a result for \ea on \lo. In general, the approach bounds 
the expected value of the potential but not of the fitness. Therefore, we also study how to derive a bound on the fitness itself, both from above and from below, by 
inverting the potential function and using tail bounds on its value. The approach uses a generalized theorem showing tail bounds for martingale 
differences, which overcomes a weakness of existing martingale difference theorems in our specific application. This generalization 
may be of independent interest.

Our results allow for giving strong fixed-budget results which were not obtainable before. For the \ea on \lo with a budget of $t = o(n^2)$ iterations, the original paper \cite{JansenZargesGECCO12} gives a lower bound of $2t/n - o(t/n)$ for the expected fitness after $t$ iterations, which we recover with a simple proof in Theorem~\ref{thm:leadingOnesWithDirectDrift}. Our theorem also allows budgets closer to the expected optimization time, where we get a lower bound of $n\ln(1+2t/n^2) - O(1)$.

For the \ea on \om, no concrete formula for a bound on the fitness value after $t$ iterations was known: The original work \cite{JansenZargesGECCO12} could only handle RLS on \om, not the \ea. The multiplicative drift theorem of \cite{LenglerSpoonerFOGA15} allows for deriving a lower bound of $n/2 + t/(2e)$ for $t = o(n)$ using a multiplicative drift constant of $(1-1/n)^n/n$. Since our drift theorem allows for variable drift, we can give a bound of $n/2 + t/(2\sqrt{e}) - o(t)$ for the \ea on \om with $t = o(n)$ (see Theorem~\ref{thm:oneMaxWithDirectDrift}). Note that \cite{LenglerSpoonerFOGA15} also gives bounds for values of $t$ closer to the expected optimization time.

Furthermore, we are not only concerned with expected values but also give strong concentration bounds. We consider the \ea on \lo and show that the fitness after $t$ steps is strongly concentrated around its expectation (see Theorem~\ref{theo:bounds-fixed-budget-final}). The  error term obtained is asymptotically smaller than in the previous work \cite{DJWZGECCO13} and the statement is also less complex.

Fixed-budget results that hold with high probability are crucial for the analysis of algorithm configurators~\cite{HallOSGECCO19}. These configurators test different algorithms for fixed budgets in order to make statements about their appropriateness in a given setting. Thus, we believe that this work also contributes to the better understanding of the strengths and weaknesses of algorithm configurators.

The remainder of the paper is structured as follows. Next we give mathematical preliminaries, covering problem and algorithm definitions as well as some well-known results from the literature which we require later. In Section~\ref{sec:directDrift} we give our direct fixed-budget drift theorems, as well as its applications to the \ea on \om and \lo. In Section~\ref{sec:variableFixedBudget} we give a variable fixed-budget drift theorem and its corollary for additive drift. We show how to apply this variable fixed-budget drift theorem to obtain very strong bounds in Section~\ref{sec:upper-tail-g}. We conclude in Section~\ref{sec:conclusions}. 

\section{Preliminaries}

The concrete objective functions we are concerned with in this paper are \om and \lo, studied in a large number of papers. These two functions are defined as follows. For a fixed natural number $n$, the functions map bit strings $x \in \{0,1\}^n$ of length $n$ to natural numbers such that
$$
\om(x) = \sum_{i=1}^n x_i
$$
is the number of $1$s in the bit string $x$ and
$$
\lo(x) = \sum_{i=1}^n \prod_{j=1}^i x_j
$$
is the number of leading $1$s in $x$ before the first $0$ (if any, $n$ otherwise).

We consider for application only one algorithm, the well-known \ea given in Algorithm~\ref{alg:ea} below.

\begin{algorithm2e}
	choose $x$ from $\{0,1\}^n$ uniformly at random\;
	\While{optimum not reached}{		$y \assign x$\;
		\For{$i = 1$ \KwTo $n$}{
			with probability $1/n$: $y_i \assign 1 - y_i$\;
		}
		\lIf{$f(y) \geq f(x)$}{$x \assign y$}
	}
\caption{The \ea for maximizing function $f$}
\label{alg:ea}
\end{algorithm2e}

For any function $f$ and $i \geq 0$, we let $f^i$ denote the $i$-times self-composition of $f$ (with $f^0$ being the identity).

\subsection{Known Results for the (1+1) EA on LeadingOnes}

We will use the following concentration result from \cite{DJWZGECCO13}, bounding the optimization time of the \ea on \lo.

\begin{theorem}[{\cite[Theorem~7]{DJWZGECCO13}}]
\label{theo:djwz}
For all $d \leq 2n^2$, the probability that the optimization time of the \ea on \lo deviates from its expectation of $(1/2)(n^2 - n)((1 + 1/(n - 1))^n - 1)$ by at least $d$, is at most $4 \exp(-d^2/(20e^2n^3))$.
\end{theorem}

 The following lemma collects some important and well-known results for the optimization process of the \ea on \lo.

\begin{lemma}
\label{lem:summary-lo}
Consider the \oneoneea on \LeadingOnes, let $x_t$ denote its search point at time~$t$ and 
$X_t=n-\LeadingOnes(x_t)$ the fitness distance. Then
\begin{enumerate}
\item 
$\E{X_{t}-X_{t+1}\mid X_{t}} = (2-2^{1-X_{t}})(1-1/n)^{n-X_{t}}/n$
 \item 
$\Prob(X_{t+1}\neq X_t\mid X_t; T>t) = (1-1/n)^{n-X_t} \frac{1}{n}$
\item 
For $j\ge 1$, 
$\Prob(X_{t+1}=X_t-j)\le \frac{1}{n}\left(\frac{1}{2}\right)^{j-1}$
\item $G_t\coloneqq X_{t}-X_{t+1}$ is a random variable with support $0,\dots,X_t$ and the following conditional distribution
on $G_t\ge 1$:
\begin{itemize}
\item $\Prob(G_t=i)=(1/2)^{i}$ for $i<X_t$
\item $\Prob(G_t=X_t)=(1/2)^{X_t-1}$
\end{itemize}
For the  moment-generating function of this $G_t$ (conditional on $G_t\ge 1$) it holds that 
\[
\E{e^{\eta G_t}\mid X_t} = \frac{(e^\eta/2)^{X_t} (1-e^\eta) + (e^\eta/2)}{1-e^\eta / 2}.
\]
\item
The expected optimization time equals $\frac{n^2-n}{2}\left(\left(1+\frac{1}{n-1}\right)^n-1\right)$, 
which is $\frac{e-1}{2}n^2 \pm O(n)$. 
\end{enumerate}
\end{lemma}

\begin{proof} The proofs of the first three statements can be found in 
 in \cite{DJWZGECCO13} and \cite[Lemma 12 of technical report]{LehreWittISAAC2014}. For the first part of the fourth statement, 
we recall from these papers that the $X_t-1$ bits after the first~$0$ are uniform and independent. Hence, the probability of observing 
$i-1<X_t$ of these so-called free-riders is $(1/2)^i$ since $i-1$ bits have to be set to~$1$ and the $i$-th bit to~$0$. 
If $i=X_t-1$ then all $i$ bits have to be set to~$1$, which has probability $(1/2)^{i}$. 

For the moment-generating function, we write (using the first part of the fourth statement) 
\[
\E{e^{\eta G_t}\mid X_t} = \sum_{j=1}^{X_t-1} \left(\frac{1}{2}\right)^j e^{\eta j} + \left(\frac{1}{2}\right)^{X_t-1} e^{\eta X_t}
\]
Since, by the geometric series,  
\[
\sum_{j=1}^{X_t-1} \left(\frac{1}{2}\right)^j e^{\eta j} = \frac{e^{\eta}/2-(e^{\eta}/2)^{X_t} }{1-e^{\eta}/2},
\]
we have 
\begin{align*}
\E{e^{\eta G_t}\mid X_t} & =  \frac{  e^{\eta}/2 -(e^{\eta}/2)^{X_t} }{1-e^{\eta}/2} + 2 \left(\frac{e^\eta}{2}\right)^{X_t} \\ 
& =  \frac{  e^{\eta}/2 -(e^{\eta}/2)^{X_t} + 2 (e^{\eta}/2)^{X_t} (1-e^{\eta}/2) }{1-e^{\eta}/2} \\ 
& = \frac{(e^\eta/2)^{X_t} (1-e^\eta) + (e^\eta/2)}{1-e^\eta / 2}.
\end{align*}
The fifth statement is due to \cite{BDNLeadingOnes}.
\end{proof}

\section{Direct Fixed-Budged Drift Theorems}\label{sec:directDrift}

In this section we give a drift theorem which gives a fixed-budget result without the detour via first hitting times. The idea is to focus on drift which gets monotonically weaker as we approach the optimum, but where being closer to the optimum is still better in terms of drift. To this end, we make the following definition.
\begin{definition}
We say that a drift function $h\colon S\to \R^{> 0}$ is \emph{greed-admitting} if $\mathrm{id} - h$ (the function $x \mapsto x - h(x)$) is monotone non-decreasing.
\end{definition}
Intuitively, this formalizes the idea that being closer to the goal is always better (\ie\ greed is good). Greed could be bad, if from one part of the search space, the drift is much higher than when being a bit closer, so that being a bit closer does not balance out the loss in drift. Note that any given differentiable $h$ is greed-admitting if and only if $h' \leq 1$. 

Typical drift functions are greed-admitting. For example, if we drift on integers, in many situations drift is less than $1$, while being closer means being at least one step closer, so being closer is always better in this sense. An example monotone process on $\{0,1,2\}$ which has a drift which is not greed-admitting is the following: $X_0$ is $2$ and the process moves to any of the states $0,1,2$ uniformly. State $0$ is the target state, from state $1$ there is only a very small probability to progress to $0$ (say $0.1$). Then it is better to stay in state $2$ than be trapped in state $1$, if the goal is to progress to state~$0$. 

We now give two different versions of the direct fixed-budget drift theorem. The first considers \emph{unlimited time}, that is, the situation where drift carries on for an arbitrary time (and does not stop once a certain threshold value is reached). This is applicable in situations where there is no end to the process (for example for random walks on the line) or when the drift eventually goes all the way down to $0$ so that the drift condition holds vacuously even when no progress is possibly any more (this is for example the case for multiplicative drift, where the drift is $\delta$ times the current value, which is naturally $0$ once $0$ has been reached). Note that this is a very strong requirement of the theorem, leading to a strong conclusion.

A special case of the following theorem is given in~\cite{LenglerSpoonerFOGA15}, where drift is necessarily multiplicative.

\begin{theorem}[Direct Fixed-Budget Drift, unlimited time]\label{thm:directDrift}
Let $X_t$, $t\ge 0$, be a stochastic process on $S \subseteq \R$, adapted to a filtration $\filtt$. Let $h\colon S \to \R^{\geq 0}$ be a \textbf{convex and greed-admitting} function such that we have the drift condition
\begin{description}
	\item[(D-ut)] $\E{X_t - X_{t+1} \mid \filtt} \geq h(X_t)$.
\end{description}	
Define $\tilde{h}(x) = x - h(x)$. Thus, the drift condition is equivalent to
\begin{description}
	\item[(D-ut')] $\E{X_{t+1} \mid \filtt} \leq \tilde{h}(X_t)$.
\end{description}	
We have that, for all $t \geq 0$,\footnote{Recall from the preliminaries that $f^i$ is the $i$-times self-composition of a function $f$.}
\[
\E{X_t \mid \filtzero} \leq \tilde{h}^t(X_0)
\]
and, in particular,
\[
\E{X_t} \leq \tilde{h}^t(\E{X_0}).
\]
\end{theorem}

\begin{proof} Note that $\tilde{h}$ is concave, since the second derivative of $\mathrm{id} - h$ is $-h''$.
We have, using this concavity of for Jensen's Inequality, for all $t$,
\begin{align*}
\E{X_{t+1} \mid \filtzero}
 & = \E{\E{X_{t+1} \mid \filtt} \mid \filtzero}\\
 & \leq \E{\E{\tilde{h}(X_{t}) \mid \filtt} \mid \filtzero}\\
 & = \E{\tilde{h}(X_{t}) \mid \filtzero}\\
 & \leq \tilde{h}(\E{X_{t} \mid \filtzero}).
 \end{align*}
Thus, the claim follows by induction with $\tilde{h}$ being non-decreasing (since $h$ is greed-admitting).
The second statement of the theorem follows with Jensen's Inequality.
\end{proof}

Now we get to the second version of the theorem, considering the more frequent case where no guarantee on the drift can be given once the optimum has been found. This weaker requirement leads to a weaker conclusion.

\begin{theorem}[Direct Fixed-Budget Drift, limited time]\label{thm:directDriftLimited}
Let $X_t$, $t\ge 0$, be a stochastic process on $S \subseteq \R$, $0 = \min S$, adapted to a filtration $\filtt$. Let $T\coloneqq \min\{t\ge 0\mid X_t=0\}$ and $h\colon S \to \R^{\geq 0}$ be a \textbf{differentiable, convex and greed-admitting} function such that $\tilde{h}'(0) \in \; ]0,1]$ and we have the drift condition
\begin{description}
	\item[(D-lt)] $\E{X_t - X_{t+1} \mid \filtt; t<T} \geq h(X_t)$.
\end{description}	
Define $\tilde{h}(x) = x - h(x)$. Thus, the drift condition is equivalent to
\begin{description}
	\item[(D-lt')] $\E{X_{t+1} \mid \filtt; t<T} \leq \tilde{h}(X_t)$.
\end{description}	
We have that, for all $t \geq 0$,
\[
\E{X_t \mid \filtzero} \leq \tilde{h}^t(X_0) + \frac{\tilde{h}(0)}{\tilde{h}'(0)}
\]
and, in particular,
\[
\E{X_t} \leq \tilde{h}^t(\E{X_0}) - \frac{\tilde{h}(0)}{\tilde{h}'(0)} \cdot \Prob(t \geq T \mid \filtzero).
\]
\end{theorem}

\begin{proof}
We let $m = - \frac{\tilde{h}(0)}{\tilde{h}'(0)}$. Recall that we assume that $0 < \tilde{h}'(0) \leq 1$. We now define a new process which mimics $(X_t)_t$, but which has to make one additional step down after reaching $0$. In order to have $0$ be the target of this new process, we will shift the old process accordingly. We let $Y_0 = X_0 + m$ and, for all $t \geq 0$,
$$
Y_{t+1} = 
\begin{cases}
	X_{t+1} + m,		&\mbox{if }t+1 < T;\\
	0,							&\mbox{else if } Y_t = 0;\\
	0,							&\mbox{else, with probability }\tilde{h}'(0);\\
	m,							&\mbox{otherwise.}
\end{cases}
$$
Intuitively, $Y_t$ behaves like $X_t + m$, but once $X_t$ hits the optimum, it will stay at $m$ until, with probability $\tilde{h}'(0)$, it jumps to $0$. We now give a drift function for this process in order to apply Theorem~\ref{thm:directDrift}.

Define $h_0$ such that, for all $x \geq 0$,
$$
h_0(x) = 
\begin{cases}
	x \cdot \tilde{h}'(0),	&\mbox{if }x < m;\\
	h(x-m),						&\mbox{otherwise.}
\end{cases}
$$
To see that $h_0$ is convex, note that it is convex on both the parts less than $m$ and above $m$; furthermore, the left- and right-derivative in $m$ coincide. Furthermore, $h_0$ is greed-admitting since it is differentiable with derivative $\tilde{h}'(0) \leq 1$ for all $x \leq m$ and with derivative at most $1$ for $x > m$ from $h$ being greed-admitting.
We see that
$
\E{Y_t - Y_{t+1} \mid \filtt} \geq h_0(X_t),
$
either by the corresponding statement about $(X_t)_t$ and $h$ or by the drift in case of $Y_t = m$ being 
$
m \cdot \tilde{h}'(0) = - \tilde{h}(0) = h(0) = h_0(m).
$
Thus, we can apply Theorem~\ref{thm:directDrift} and get
$
\E{Y_t \mid \filtzero} \leq \tilde{h_0}^t(Y_0).
$
By induction we get
$
\tilde{h_0}^t(Y_0) = \tilde{h_0}^t(X_0 + m) = \tilde{h}^t(X_0) + m.
$
From $X_t \leq Y_t - m \cdot \mathds{1}[t<T]$ we thus get
\begin{align*}
\E{X_t \mid \filtzero} 
 & \leq \E{Y_t \mid \filtzero} - m \cdot \Prob(t<T \mid \filtzero)\\
 & \leq \tilde{h_0}^t(Y_0) - m \cdot \Prob(t<T \mid \filtzero)\\
 & = \tilde{h}^t(X_0) + m - m \cdot \Prob(t<T \mid \filtzero)\\
 & = \tilde{h}^t(X_0) +  m \cdot \Prob(t \geq T \mid \filtzero).
\end{align*}
This concludes the proof.
\end{proof}

With the following theorem we give a general way of iterating a greed-admitting function, as necessary for the application of the previous two theorems. From this we can see the similarity of this approach to the method of variable drift theory where the inverse of $h$ is integrated over, see Theorem~\ref{theo:variable-fixed-budget-expect} and the discussion about drift theory in general in~\cite{Lengler2020}.

\begin{theorem}\label{thm:estimateTildeH}
Let $h$ be greed-admitting and let $\tilde{h} = \mathrm{id} - h$. Then we have, for all starting points $n$ and all target points $m < n$ and all time budgets $t$,
$$
\mbox{if } t \geq \sum_{i=m}^{n-1} \frac{1}{h(i)} \mbox{ then } \tilde{h}^t(n) \leq m.
$$
\end{theorem}

\begin{proof} The idea of this proof is that each application of $\tilde{h}$ on some value $\geq i$ gains at least $h(i)$, so gaining this amount at least $1/h(i)$ times decreases a value of at most $i+1$ to a value of at most $i$. A simple induction then gives the claimed result.
More formally, for all $i$ and all $t \geq 1/h(i)$, we have that $\tilde{h}^t(i+1) \leq i$. Thus we inductively get, for all $k$, if $t \geq \sum_{i=n-k}^{n-1} 1/h(i)$, then $\tilde{h}^t(n) \leq n-k$. Using the induction statement for $k = n-m$ gives the result.
\end{proof}

\subsection{Application to OneMax}

In this section we show how we can apply Theorem~\ref{thm:directDrift} by using the optimization of the \ea on \om as an example (where we have multiplicative drift).

\begin{theorem}\label{thm:oneMaxWithDirectDrift}
Let $V_t$ be the number of $1$s which the \ea on \OneMax has found after $t$ iterations of the algorithm. Then we have, for all $t$,
$$
\E{V_t} \geq 
\begin{cases}
\frac{n}{2} + \frac{t}{2\sqrt{e}}-O(1),				&\mbox{if }t = O(\sqrt{n});\\
\frac{n}{2} + \frac{t}{2\sqrt{e}}(1-o(1)),			&\mbox{if }t = o(n).
\end{cases}
$$
Furthermore, for all $t$, we have $\E{V_t} \geq n(1 - \exp(-t/(en))/2)$.
\end{theorem}

\begin{proof} We can apply the unlimited time theorem (Theorem~\ref{thm:directDrift}) to the \ea on \om by using the drift function $h(x) = (1-1/n)^{n-x} \frac{x}{n}$. This function is convex and greed-admitting, and it also applies in case the process already reached the optimum of $0$ (since $h(0)=0$). We now need to estimate $\tilde{h}^t$.

In order to apply Theorem~\ref{thm:estimateTildeH}, we estimate as follows (using estimates for the harmonic sum which use $c = o(n)$).
\begin{align*}
 & \sum_{i=n/2-c}^{n/2} \frac{1}{h(i)}
    = \sum_{i=n/2-c}^{n/2} (1-1/n)^{i-n} \frac{n}{i}\\
 & \leq n \sum_{i=n/2-c}^{n/2} (1-1/n)^{-c-n/2} \frac{1}{i}\\
 & \leq n (1-1/n)^{-c-n/2} \sum_{i=n/2-c}^{n/2} \frac{1}{i}\\
 & \leq n \exp( (c+n/2)/n ) \left(\ln(n/2) - \ln(n/2-c) + O(1/n)\right)\\
 & = n \sqrt{e} e^{c/n} \left(- \ln\left((n/2-c)/(n/2)\right) + O(1/n)\right)\\
 & = n \sqrt{e} e^{c/n} \left(- \ln(1 - 2c/n) + O(1/n)\right)\\
 & \leq n \sqrt{e} (1+c/n + c^2/n^2) \left(2c/n + O(1/n)\right)\\
 & = 2 \sqrt{e} (1+c/n + c^2/n^2) \left(c + O(1)\right).
\end{align*}
For $c = =(\sqrt{n})$ the last term is at most $2\sqrt{e}c+O(1)$ and for $c = o(n)$ it is $2\sqrt{e}c(1+o(1))$.
Thus, we get the claimed bounds with Theorem~\ref{thm:estimateTildeH}.
Regarding the ``furthermore'' clause, we argue more directly about $\tilde{h}^t$ by observing that, for all $x$, $\tilde{h}(x) \leq (1-1/(en))x$ and thus, by a straightforward induction (similar to the proof of~\cite[Theorem~1]{LenglerSpoonerFOGA15}) we get, for all $x,t$, $\tilde{h}^t(x) \leq (1-1/(en))^tx \leq e^{-t/(en)}x$. This gives the desired result with initial state $n/2$.
\ignore{Let now $m = n/2-c$ and let $p$ be such that $m = pn$. Then we continue the above estimation, instead of estimating the logarithm linearly, as follows.
\begin{align*}
 & n \sqrt{e} e^{c/n} \left(- \ln(1 - 2c/n) \pm O(1/n)\right)\\
 & \leq n e \left(\ln(1/2p) \pm O(1/n)\right).
\end{align*}
We want to argue with Theorem~\ref{thm:estimateTildeH} again. We see that this last expression is less than some $t$ if $p > \exp(-t/(en))/2$, which shows the ``furthermore'' clause. 
}\end{proof}

\subsection{Application to LeadingOnes}

In this section we want to use Theorem~\ref{thm:directDriftLimited} to the progress of the \ea on \LeadingOnes. The result is summarized in the following theorem.
\begin{theorem}\label{thm:leadingOnesWithDirectDrift}
Let $V_t$ be the number of leading $1$s which the \ea on \LeadingOnes has found after $t$ iterations of the algorithm. We have, for all $t$,
$$
\E{V_t} \geq 
\begin{cases}
\frac{2t}{n} - O(1),										&\mbox{if }t=O(n^{3/2});\\
\frac{2t}{n} \cdot (1 - o(1)),							&\mbox{if }t=o(n^2);\\
n\ln(1+\frac{2t}{n^2}) - O(1),							&\mbox{if }t \leq \frac{e-1}{2} n^2 - n^{3/2}.
\end{cases}
$$
\end{theorem} 

\begin{proof} For the derivation of fitness drift of the \ea on \LeadingOnes, see the first item of Lemma~\ref{lem:summary-lo}. We want to use Theorem~\ref{thm:directDriftLimited} to get our fixed-budget result.

However, in order to make our analysis, we artificially change the fitness value of the all-$1$s string to $n+1$ (rather than $n$). The result of this change is in the expected fitness gain: if any fitness is gained at all, the total gain is usually (for plain \LeadingOnes) $1$ plus the number of ``free rider'' bits, additional bits after the first that happen to be set to $1$. There cannot be an arbitrary number of them (since the bit string is finite -- of size $n$), so the total expected number of bits gained is slightly less than $2$: it is $2-(1/2)^{n-1-x}$. By artificially changing the fitness value of the perfect string to $n+1$ we now have an expected value increase of at least $2$, as long as the best bit string has not been found (conditional on making an improvement at all).

Note that this change of the fitness value of the all-$1$s string changes the final result only by at most $1$, which is consumed by the $O$-notation.

Thus, we can use the drift function 
$$
h\colon [0,n] \to \R^{> 0}, x \mapsto \left(1-\frac{1}{n}\right)^{n-x}\frac{2}{n}.
$$
We have that $h$ is greed-admitting (since the drift changes only very little, it would have to change by more than $1$ between two distance $1$ states) and convex (since the exponential function is convex). Note that without the artificial change mentioned above the actual drift would not have been convex.

We let $\tilde{h}(x) = x-h(x)$. We want to aplly Theorems~\ref{thm:directDriftLimited} so we note that
$$
\tilde{h}(0) = - \left(1-\frac{1}{n}\right)^{n}\frac{2}{n} = - \Theta(1/n)
$$
and
$$
\tilde{h}'(0) = 1 + \left(1-\frac{1}{n}\right)^{n}\frac{2}{n} \ln(1-1/n) = 1 - o(1).
$$
In order to estimate the $t$-fold application of $\tilde{h}$ on $0$ we use Theorem~\ref{thm:estimateTildeH}. Let $c = n-m$. We have
\begin{align*}
\sum_{i=m}^{n-1} \frac{1}{h(i)}
 & = \sum_{i=m}^{n-1} \left(1-\frac{1}{n}\right)^{i-n}\frac{n}{2}\\
 & = \frac{n}{2} \sum_{j=1}^{c} \left(1-\frac{1}{n}\right)^{-j}\\
 & = \frac{n}{2} \left( \frac{\left(1-\frac{1}{n}\right)^{-c-1}-1}{1/(1-1/n)-1} - 1\right)\\
 & = \frac{n(n-1)}{2} \left( \left(1-\frac{1}{n}\right)^{-c-1} - 1 - \frac{1}{n-1}\right)\\
 & \leq \frac{n(n-1)}{2} \left( \exp\left(\frac{c+1}{n}\right) - 1 - \frac{1}{n-1}\right).
\end{align*}
From this we already get the third and most general claimed bound using the concentration bound given in Theorem~\ref{theo:djwz} with an appropriate $d = \Theta(n^{3/2})$, where the probability of having reached the optimum is some constant.

We can continue the estimates as
\begin{align*}
 & \leq \frac{n(n-1)}{2} \left( \frac{c+1}{n} + \left(\frac{c+1}{n}\right)^2 - \frac{1}{n-1}\right)\\
 & \leq \frac{n-1}{2} \left( c+1 + \frac{(c+1)^2}{n} - \frac{1}{n-1}\right).
\end{align*}
This term is at most $(c+1)n/2+o(n)$ for $c = o(\sqrt{n})$; and $cn/2+O(n)$ for $c=\Theta(\sqrt{n})$ and $(1+o(1))cn/2$ for $c=o(n)$.

\ignore{
, first note that $\tilde{h}^t(n)$ is lower bounded by $n-2t/n$. This means that, for $t \leq n/2$ steps, $\tilde{h}^t(0) \geq n-1$, for the next $n/2$ steps, so for $t \in [2/n+1..4/n]$, we have $\tilde{h}^t(0) \geq n-2$. In general, for all $i$, for $t \in [2i/n+1..2(i+1)/n]$, we have $\tilde{h}^t(0) \geq n-i-1$.

Since $h$ is monotone decreasing we get, for all $i$, for $t+1 \in [2i/n+1..2(i+1)/n]$
$$
\tilde{h}^{t+1}(n) = \tilde{h}^{t}(n) - h(\tilde{h}^{t}(n)) \leq \tilde{h}^{t}(n) - h(n-i-1).
$$
We now use this idea iteratively and get, for all $c$ by considering phases of length $n/2$,
\begin{align*}
n- \tilde{h}^{cn/2}(n)
 & \geq \sum_{t=1}^{cn/2} h(n-\lceil 2t/n \rceil)\\
 & \geq \sum_{i=1}^c \frac{n}{2} \cdot h(n-i)\\
 & = \sum_{i=1}^c \frac{n}{2} \cdot \left(1-\frac{1}{n}\right)^{i}\frac{2}{n}\\
 & = \sum_{i=1}^c \left(1-\frac{1}{n}\right)^{i}\\
 & = \frac{1-\left(1-\frac{1}{n}\right)^{c}}{1-(1-1/n)} - 1\\
 & = n\left( 1-\left(1-\frac{1}{n}\right)^{c}\right) - 1\\
 & \geq n\frac{c/n}{1+c/n} - 1\\
 & = c\frac{1}{1+c/n} - 1.
\end{align*}

This term is at least $c-o(1)$ for $c = o(\sqrt{n})$; $c-O(1)$ for $c=\Theta(\sqrt{n})$ and $c+o(c)$ for $c=o(n)$.

This shows that, essentially, after $t$ iterations we expect to have $\frac{2t}{n}$ leading ones.
}

We use Theorem~\ref{theo:djwz} again to see that the \ea on \LeadingOnes is done in $o(n^2)$ steps with probability at most $\exp(-n)$, which suffices to get the first two desired bounds with the help of Theorems~\ref{thm:directDriftLimited} and~\ref{thm:estimateTildeH}.
\end{proof}

\section{Variable Drift Theorem for Fixed Budget}
\label{sec:variableFixedBudget}
We now turn to an alternative approach to derive fixed-budget results via drift analysis. 
Our method is based on variable drift analysis that was introduced 
to the analysis of randomized search heuristics by Johannsen~\cite{Johannsen2010}. 
Crucially, variable drift analysis applies a specific transformation, the so-called 
potential function~$g$, to the state space. Along with bounds on the hitting times,
we obtain the following theorem estimating the expected value of the potential function 
after~$t$ steps. Subsequently, we will discuss how this information can be used 
to analyze the untransformed state.

\begin{theorem}
\label{theo:variable-fixed-budget-expect}
Let $X_t$, $t\ge 0$, be a stochastic process, adapted to a filtration $\filtt$, on 
$S\coloneqq \{0\}\cup \R^{\ge \xmin}$ for some $\xmin>0$. Let $T\coloneqq \min\{t\ge 0\mid X_t=0\}$ and 
$h\colon S\to \R^{> 0}$ be a non-decreasing function such that $\E{X_t-X_{t+1} \mid \filtt; t<T} \ge h(X_t)$.  Define
$g\colon S\to \R$ by 
\[
g(x)\coloneqq 
\begin{cases}
\frac{\xmin}{h(\xmin)} + \int_{\xmin}^x \frac{1}{h(z)} \,\mathrm{d}z & \text{if $x\ge \xmin$}\\
0 & \text{otherwise}
\end{cases}.  
\]
Then it holds that
\[
\E{g(X_t)\mid \filtzero} \le g(X_0) - 
\sum_{s=0}^{t-1} \Prob(s<T).  
\]
\end{theorem}

\begin{proof} Since $h$ is non-decreasing, $g$ is concave. We 
claim that the drift  of the $g$-value is bounded from below by~$1$, 
formally
\begin{equation}
\E{g(X_t)-g(X_{t+1})\mid \filtt; t<T} \ge 1\label{eq:driftone}
\end{equation}
To prove the claim, we use standard arguments from the proof
 of the variable drift theorem for expected hitting times. 
Expanding the definition of~$g$, we obtain
\begin{align*}
\E{g(X_t)-g(X_{t+1})\mid \filtt; t<T} & = 
 \int_{\xmin}^{X_t} \frac{1}{h(z)} \,\mathrm{d}z \\
 & \quad - \Ebig{\int_{\xmin}^{X_{t+1}} \frac{1}{h(z)} \,\mathrm{d}z\mid \filtt}.
\end{align*}
By Jensen's inequality and the concavity of~$g$, we have 
\[
\E{g(X_t)-g(X_{t+1})\mid \filtt;t<T} \ge  
 \int_{\xmin}^{X_t} \frac{1}{h(z)} \,\mathrm{d}z 
 - \int_{\xmin}^{\E{X_{t+1}\mid \filtt}} \frac{1}{h(z)} \,\mathrm{d}z,
\]
which, since $\E{X_{t+1}\mid \filtt;t<T} \le X_t-h(X_t)$, 
is at least 
\[
  \int_{X_t-h(X_t)}^{X_t} \frac{1}{h(z)} \,\mathrm{d}z \ge 
	 \int_{X_t-h(X_t)}^{X_t} \frac{1}{h(X_t)} \,\mathrm{d}z = 1, 
\]
where the inequality used that $h(z)$ in non-decreasing.

We proceed by estimating $\E{g(X_t)}$ in an inductive fashion. 
By the law of total probability,
\[
\E{g(X_1)\mid \filtzero} = g(X_0) -\Prob(0<T)\left(g(X_0)-\E{g(X_1)\mid \filtzero;0<T}\right)
\]
so with \eqref{eq:driftone},
\[
\E{g(X_1)\mid \filtzero} \le g(X_0) -\Prob(T>0)  . 
\]
 Noting that
\begin{align*}
& \E{g(X_t)\mid \filtzero}  \\ 
& \quad = \E{ g(X_{t-1}) - 
\E{g(X_{t-1})-g(X_t)\mid \filttminone;t-1<T} \mid \filtzero},
\end{align*}
we get by the induction hypothesis and \eqref{eq:driftone}  that 
\begin{align*}
 \E{g(X_t)\mid \filtzero} & \le \E{ g(X_{t-1})\mid \filtzero} - 
\Prob(t-1<T)  \\
 &  \le g(X_0) - \sum_{s=0}^{t-2}\Prob(s<T) - \Prob(t-1<T)  
\end{align*}
altogether
\[
\E{g(X_t)\mid \filtzero} \le g(X_0) - 
\sum_{s=0}^{t-1} \Prob(s<T).  
\]
as suggested.
\end{proof}

\subsection{Additive Drift as Special Case}
\label{sec:timo}

A special case of variable drift is additive drift, when the drift function $h$ is constant.

\begin{theorem}
\label{theo:additive-fixed-budget-expect}
Let $X_t$, $t\ge 0$, be a stochastic process, adapted to a filtration $\filtt$, on 
$S\coloneqq \R^{\ge 0}$. Let $T\coloneqq \min\{t\ge 0\mid X_t=0\}$ and 
$\delta \in \R^{> 0}$ be such that $\E{X_t-X_{t+1} \mid \filtt; t<T} \ge \delta$. 
Then we have
\[
\E{ X_t \mid \filtzero} \le X_0 - 
\delta \sum_{s=0}^{t-1} \Prob(s<T).  
\]
\end{theorem}
The theorem is a corollary to Theorem~\ref{theo:variable-fixed-budget-expect} by using $\xmin = \delta$, the smallest value for which the condition of a drift of at least $\delta$ can still be obtained, and thus the smallest value (other than $0$) that the process can attain.

As a sample application, we can now derive an estimate of the best value found by the \ea on \LeadingOnes within $t$ steps, using the concentration result from \cite{DJWZGECCO13} given in Theorem~\ref{theo:djwz}.

\begin{theorem}
\label{theo:lo-direct-lower}
Let $V_t$ be the number of  leading $1$s which the \ea on \LeadingOnes has found after $t$ iterations of the algorithm. Then, for all $t \leq \frac{e-1}{2}n^2 - n^{3/2}\log(n)$, we have
$$
\E{V_t} \geq \frac{2t}{en} - O(1).
$$
\end{theorem}

\begin{proof} We drift on the potential which assigns each bit string its number of leading ones, except for the all-$1$ string which has a potential of $n+1$. A quick computation shows that this leads to an expeced increase in potential of $2$, conditional on the potential increasing at all (without the ``$+1$'' for the all-$1$-string, it would have been slightly less than $2$). We use drift on this potential, which is, for all current potential values $x < n$, now lower bounded by 
$$
\left(1-\frac{1}{n}\right)^x \frac{2}{n} \geq \frac{2}{en}.
$$
Thus, the result follows with Theorem~\ref{theo:additive-fixed-budget-expect} and the concentration bound given in Theorem~\ref{theo:djwz}.
\end{proof}

Note that the result was proven very easily with a direct application of the additive version of the fixed-budget drift theorem in combination with a strong result on concentration. The price paid for this simplicity is that the lead constant in this time bound is not tight, as can be seen by comparing with the results given in Theorem~\ref{thm:leadingOnesWithDirectDrift}.

\section{Variable Drift and Concentration Inequalities}
\label{sec:upper-tail-g}
The expected $g(X_t)$-value derived in Theorem~\ref{theo:variable-fixed-budget-expect} is 
not very useful unless it allows us to make conclusions  on the underlying $X_t$-value. The previous 
application in Section~\ref{sec:timo} only gives tight bounds in case that the drift 
is more or less constant throughout the search space. This is not the case for \om and \lo 
where the drift increases with the distance to the optimum (\eg, for \om the drift is $\Theta(1/n)$ at distance~$1$ and $\Theta(1)$ as distance~$n/2$; for 
\lo the drift can vary by a term of roughly~$e$). Hence, looking back into 
Theorem~\ref{theo:variable-fixed-budget-expect}, we now are interested in characterizing 
$g(X_t)$ more precisely than just in terms of expected value. 
If we manage to establish 
concentration of $g(X_t)$ then we can (after inverting~$g$) derive a maximum of the $X_t$-value that 
holds with sufficient probability. Our main result achieved along this path is the following one.

\begin{theorem}
\label{theo:bounds-fixed-budget-final}
Let $V_t$ be the number of  leading $1$s which the \ea on \LeadingOnes has found after $t$ iterations.
Then for $t=\omega(n\log n)$ and $t\le (e-1)n^2/2 - cn^{3/2}\sqrt{\log n}$, where $c$  is a sufficiently large constant 
the following statements hold. (a) With probability at least $1-1/n^3$,
\begin{align*}
-n\ln\!\left(1-2t/n^2+O(\sqrt{t\log n}/n^{3/2})\right) \;&\le V_t\\ 
-n \ln\! \left(1-2t/n^2-O(\sqrt{t\log n}/n^{3/2})\right) \;&\geq V_t.
\end{align*}
(b) $\E{V_t} = -n\ln(1-2t/n^2+O(\sqrt{t\log n}/n^{3/2}))$.
\end{theorem}

To compare with previous work, we note that
 the  additive error is $O(\sqrt{t\log n}/n^{1/2})$. This is asymptotically smaller than the additive error term 
of order $\Omega(n^{3/2+\epsilon})$ that appears in the fixed-budget statements of \cite{DJWZGECCO13} and 
moreover, it depends on~$t$. Also, we think 
that the formulation of our statement is 
 less complex than in that paper.

The proof of Theorem~\ref{theo:bounds-fixed-budget-final} overcomes several technical challenges. The first 
idea is to apply established concentration inequalities for stochastic processes.  
Since (after a reformulation discussed below)
 the process of $g$\nobreakdash-values describes 
a (super)martingale, it is natural to take the method of bounded martingale differences. However, since there 
is no ready-to-use theorem for all our specific martingales, we present a generalization of martingale concentration 
inequalities in the following subsection Section~\ref{sec:theo-mar-differences}. The concrete 
application
is then given in Sections~\ref{sec:applic-concentration} onwards.

\subsection{Tail Bounds for Martingale Differences}
\label{sec:theo-mar-differences}
The classical method of bounded martingale differences \cite{McDiarmid1998} considers a (super)martingale $Y_t$, $t\ge 0$, 
and its corresponding martingale differences $D_t=Y_{t+1}-Y_{t}$. Given certain boundedness conditions 
for $D_t$ (\eg, that $\card{D_t}\le c$ for a constant~$c$ almost surely), it is shown that the sum 
of martingale differences 
$\sum_{i=0}^{t-1} D_i = Y_t - Y_0$ does not deviate much from its expectation $Y_0$ (resp.\ is not much 
bigger in the case of supermartingales). This statement remains essentially true if $D_t$ is allowed to have unbounded support but 
exhibits a strong concentration around its expected value. Usually, this concentration is formulated in terms 
of a so-called \emph{subgaussian} 
(or, similarly, \emph{subexponential}) property \cite{KoetzingTailBoundAlgorithmica, FanExponentialMartingales2015}. 
Roughly speaking, this property requires that 
the moment-generating function (mgf.) of the differences can be bounded as 
$\E{e^{\lambda D_t}\mid \filtt} \le e^{\lambda^2 \nu_t^2/2}$ for a certain parameter $\nu_t$ and all $\lambda<1/b_t$, 
where $b_t$ is another parameter. In particular, the bound has to remain true when $\lambda$ becomes arbitrarily small.

In one of our concrete applications of the martingale difference technique, 
the inequality $\E{e^{\lambda D_t}\mid \filtt} \le e^{\lambda^2 \nu_t^2/2}$ is true for certain 
values of $\lambda$ below a threshold $1/b^*$, but does not hold if $\lambda$ is much smaller than $1/b^*$. We therefore 
show that 
the concentration of the sums of martingale differences to some extent remains
 true if the inequality only holds for $\lambda\in[1/a^*,1/b^*]$ where $a^*>b^*$ is another parameter. The approach 
uses well-known arguments for the proof of concentration inequalities. Here, we were inspired by the notes 
\cite{WainwrightNotes2015}, which require the classical subexponential property, though.

\begin{theorem}
\label{theo:martingale-diff-subexp-weak}
Let $Y_t$, $t\ge 0$, be a supermartingale, adapted to a filtration $\filtt$, and 
let $D_t=Y_{t+1}-Y_{t}$ be the corresponding martingale differences. Assume that there 
are $0<b_2<b_1 \le \infty$ and a sequence $\nu_t$, $t\ge 0$, such that 
for $\lambda\in [1/b_1,1/b_2]$ it holds 
that $\E{e^{\lambda D_t}\mid \filtt} \le e^{\lambda^2 \nu_t^2/2}$. Then  for all $t\ge 0$ it holds that 
\[  
\Prob(Y_t - Y_0 \ge d) \le 
\begin{cases}
e^{-d/(2b_2)} & \text{ if $d\ge \frac{\sum_{i=0}^{t-1} \nu_i^2}{b_2}$} \\
e^{-d^2/(2\sum_{i=0}^{t-1} \nu_i^2)} & \text{ if $\frac{\sum_{i=0}^{t-1} \nu_i^2}{b_1} \le d < \frac{\sum_{i=0}^{t-1} \nu_i^2}{b_2}$}
\end{cases}
\]
The theorem holds analogously for submartingales with respect to the tail bound
$\Prob(Y_t - Y_0 \le -d)$. 
\end{theorem}

\begin{proof} We consider the mgf.\ of the sum $S_t\coloneqq \sum_{i=0}^{t-1} D_i$. Using the usual Chernoff-type approach, 
we have for all $\lambda\ge 0$ that 
\begin{align*}
\Prob(S_t \ge d) & = \Prob( e^{\lambda S_t} \ge e^{\lambda d} ) 
 \le e^{-\lambda d} \E{e^{\lambda S_t} \mid \filtzero}
\end{align*}
To bound the last mgf., we note that by the law of total expectation, 
\begin{align*}
\E{e^{\lambda S_t} \mid \filtzero} & = \Ebig{e^{\lambda \sum_{i=0}^{t-2} D_i} \cdot \E{e^{\lambda D_{t-1}} \mid \filttminone} \mid \filtzero  } 
\\ 
& \le \E{e^{\lambda S_{t-1}}\mid \filtzero} e^{\lambda^2 \nu_{t-1}^2/2},
\end{align*}
where the last inequality used the assumption from the theorem, which is valid 
after we assume $1/b_1 \le \lambda \le 1/b_2$. Iterating this argument, we obtain
\[
\E{e^{\lambda S_t} \mid \filtzero} \le e^{\lambda^2 \sum_{i=0}^{t-1} \nu_i^2/2},
\]
hence 
\[
\Prob(S_t \ge d) \le e^{-\lambda d} e^{\lambda^2 \sum_{i=0}^{t-1} \nu_i^2/2}.
\]
In the following, we write $V_t\coloneqq \sum_{i=0}^{t-1} \nu_i^2$.  
We now distinguish between the two cases for $d$ displayed in the lemma. If 
$d\ge V_t/b_2$, 
this leads to 
\[
\Prob(S_t \ge d) \le e^{-\lambda d} e^{\lambda^2 d b_2/2},
\]
which, choosing $\lambda\coloneqq 1/b_2$, yields
\[
\Prob(S_t \ge d) \le e^{-d/b_2 + (1/b_2)^2 d b_2/2} = e^{-d/(2b_2)}.
\]
If $V_t/b_1 \le d < V_t/b_2$ we choose 
$\lambda\coloneqq d/V_t \in [1/b_1,1/b_2)$. Then
\[
\Prob(S_t \ge d) \le  e^{-\lambda d + \lambda^2 V_t/2} = e^{-\frac{d}{V_t} d + \frac{d^2}{V_t^2}\frac{V_t}{2}} = e^{-d^2/(2V_t)},
\]
which, after substituting $V_t$, proves the theorem.
\end{proof}

\subsection{Preparing an Upper Tail Bound via the  Martingale Difference Method}
\label{sec:applic-concentration} 

We now return to Theorem~\ref{theo:variable-fixed-budget-expect} and would like to show concentration 
of $g(X_t)$ in order to show a bound for $X_t$ that holds with sufficiently high probability. 
Note that by the statement of the theorem, we immediately have that  
$Y_t\coloneqq g(X_t)+\sum_{s=0}^{t-1}\Prob(T>s)$ is a supermartingale. By bounding the probability 
of $Y_t\ge d$ for arbitrary $t\ge 0$ and $d\ge 0$, \ie, establishing concentration of the supermartingale $Y_t$ 
via  Theorem~\ref{theo:martingale-diff-subexp-weak},
 and inverting~$g$, we will obtain a bound 
on the probability of the event $g(X_t) \ge \E{g(X_t)}$. 

As we want to prove Theorem~\ref{theo:bounds-fixed-budget-final}, the application is again the \oneoneea on the \LeadingOnes function, 
so $X_t=n-\LeadingOnes(x_t)$ is the fitness distance of the \LeadingOnes-value at time~$t$ from the target.

Defining $h(X_t) \coloneqq \E{X_{t}-X_{t+1}\mid X_{t}}$ according to Lemma~\ref{lem:summary-lo} 
and $g(X_t)=1/h(1) + \int_{1}^{X_t} 1/h(z)\,\mathrm{d}z$ 
according to Lemma~\ref{theo:variable-fixed-budget-expect}, we will establish the following bound 
on the moment-generating function (mgf.) of the drift of our concrete $g$.

\begin{lemma}
\label{lem:mgf-uppertail}
Let $T$ denote the optimization time of the \oneoneea on \LeadingOnes. If  $\lambda\le 1/(2en)$ then 
 $\E{e^{\lambda(g(X_{t+1})-g(X_t)+\Prob(T>t))}\mid X_t} = e^{O(\lambda^2 n)}$.
\end{lemma}

\begin{proof}
We write $Y_t=g(X_t)+\sum_{s=0}^{t-1}\Prob(T>s)$ and  $D_t=Y_{t+1}-Y_{t}=g(X_{t+1})-g(X_t)+\Prob(T>t)$. Without loss of generality the process stops from time~$T$ on so that 
$g(X_{t+1})-g(X_t)=0$ for all $t\ge T$. Hence, 
conditional on $T\le t$ we have $\E{e^{\lambda D_t}} = e^{\lambda 0 + 0} = 1 \le e^{\lambda^2 n}$. We now consider the
interesting case that $T>t$. Then 
\begin{align*}
& \E{e^{\lambda D_t}\mid X_{t}} = \E{e^{\lambda(g(X_{t+1})-g(X_t)+1)}\mid X_t}
\end{align*} 
since $\Prob(T>t)=1$ on our condition. To bound this mgf., we shall exploit that 
\[ g(X_{t+1})-g(X_t) \le \frac{X_{t+1}-X_t}{h(X_t)}
\] (by the concavity of $g$, noting that the difference is negative). Hence, 
$g(X_{t+1})-g(X_t)$ is stochastically dominated by $\frac{X_{t+1}-X_t}{h(X_t)}$.

Applying the law of total probability with respect to an improving step (using the second item from 
Lemma~\ref{lem:summary-lo}) and writing $G_t=X_t-X_{t+1}$, we obtain 
\begin{align}
\notag
\E{e^{\lambda(g(X_{t+1})-g(X_t)+1)}\mid X_t} & \le  
\left(1-\frac{1}{n}\right)^{n-X_t} \frac{1}{n} \E{e^{\frac{\lambda}{h(X_t)} (-G_t) + \lambda}\mid X_t}\\
& \qquad + \left(1- \left(1-\frac{1}{n}\right)^{n-X_t}\frac{1}{n} \right) e^{\lambda}\label{eq:momgen1}
\end{align}
We will write $\eta = \frac{\lambda}{h(X_t)}$ in the following and assume $0\le \eta = o(1)$, which implies 
the same for $\lambda$. Using the well-known inequalities $1+x+x^2 \ge e^{x}\ge 1+x$ and $1-x+x^2\ge e^{-x}\ge 1-x$ 
for $0\le x\le 1$, we obtain the following bound on \eqref{eq:momgen1}:
\begin{align}
\notag 
& \E{e^{\lambda(g(X_{t+1})-g(X_t)+1)}\mid X_t} \\
& \leq \notag 
\left(1-\frac{1}{n}\right)^{n-X_t} \frac{1}{n} (1+\lambda+\lambda^2) \E{e^{-\eta G_t }\mid X_t}\\
 & \qquad + \left(1- \left(1-\frac{1}{n}\right)^{n-X_t} \frac{1}{n} \right) (1+\lambda+\lambda^2).
\label{eq:momgen2}\end{align}
The most challenging part is now to bound $\E{e^{-\eta G_t }\mid X_t}$ (still conditional on $G_t\ge 1$)  in such a way that 
we obtain an estimate depending on $\lambda$ only. We use the shorthand $X_t=i$ from now on.
 From Lemma~\ref{lem:summary-lo} and again using the inequalities with $e^x$ we have
\begin{align*}
\E{e^{-\eta G_t }\mid X_t; G_t\ge 1} & = \frac{(e^{-\eta}/2)^i (1-e^{-\eta}) + (e^{-\eta}/2)}{1-e^{-\eta} / 2} \\
& \le \frac{(\frac{1-\eta+\eta^2}{2})^i (1-(1-\eta)) + \frac{1-\eta+\eta^2}{2}}{1-(1/2)(1-\eta+\eta^2)} \\
& =  \frac{(\frac{1-\eta+\eta^2}{2})^i \eta + \frac{1-\eta+\eta^2}{2}}{(1/2)+\eta/2-\eta^2/2} \\
& \le \frac{(\frac{1}{2})^{i-1} \eta + 1-\eta+\eta^2}{1+\eta-\eta^2} \\
& \le \left(\frac{1}{2}\right)^{i-1} \eta + \frac{1-\eta+\eta^2}{1+\eta-\eta^2} \\
& = \left(\frac{1}{2}\right)^{i-1} \eta + 1-\frac{2\eta-2\eta^2}{1+\eta-\eta^2} \\ 
& = 1 + \left(\frac{1}{2}\right)^{i-1} \eta - 2\eta +O(\eta^2)
\end{align*}
We are now ready to substitute $h(i) = (1-1/n)^{n-i}(1/n) (2-\left(\frac{1}{2}\right)^{i-1})$ (Lemma~\ref{lem:summary-lo}) in 
$\eta = \lambda/h(i)$ and obtain from the preceding estimate that 
\begin{align*}
\E{e^{-\eta G_t }\mid X_t; G_t\ge 1} & \le 
1 + \eta \left(\left(\frac{1}{2}\right)^{i-1}  - 2\right) +O(\eta^2)  \\ 
& = 1 + \frac{\lambda n \left(1-\frac{1}{n}\right)^{i-n}}{2-\left(\frac{1}{2}\right)^{i-1}}
\left(\left(\frac{1}{2}\right)^{i-1}  - 2\right) +O(\eta^2)  \\ 
& = 1 - \lambda n \left(1-\frac{1}{n}\right)^{i-n} + O(\lambda^2 n^2),
\end{align*}
 where the last inequality used $1/h(i)=O(n)$. Plugging this back into \eqref{eq:momgen2}, we 
have 
\begin{align}
&\notag \E{e^{\lambda(g(X_{t+1})-g(X_t)+1)}\mid X_t=i} \\
&\notag  = 
\left(1-\frac{1}{n}\right)^{n-i} \frac{1}{n} (1+\lambda+\lambda^2)  \left(1 - \lambda n \left(1-\frac{1}{n}\right)^{i-n} + O(\lambda^2 n^2)\right)\\\notag
&\notag \qquad + \left(1- \left(1-\frac{1}{n}\right)^{n-i}\frac{1}{n} \right) (1+\lambda+\lambda^2) \\
& = 1 + \lambda + \lambda^2 - \lambda + O(\lambda ^2 n) \le e^{\lambda^2 + O(\lambda^2 n)} = e^{O(\lambda^2 n)}  
\label{eq:momgen3}
\end{align}
for $\lambda\le 1/(2en)$. 
The restriction on $\lambda$
 follows from the fact that $\eta\le 1$ has to be ensured, which follows for $\lambda\le 1/(2en)$ from 
the bound $h(x)\le 2en$.
\end{proof}

Looking into Theorem~\ref{theo:martingale-diff-subexp-weak} the required subexponential property of the martingale difference~$D_t$ 
has been proven with $\nu = O(\sqrt{n})$ and $\lambda \le 1/(2en) = 1/b^*$.
 Before we formally apply this lemma, we also establish concentration in the other direction.

\subsection{Preparing a Lower Tail Bound}
\label{sec:lower-tail-g}
We will now complement the upper tail bound for~$g$ that 
we prepared in the previous subsection  with a lower tail bound. 
The aim is again to apply Theorem~\ref{theo:martingale-diff-subexp-weak}, this time with respect to the sequence 
$Y_t=g(X_t)+\sum_{s=0}^{t-1}\Prob(T>s)+r(t,n)$, 
 where $X_t=n-\LeadingOnes(x_t)$ is still the fitness distance of the \LeadingOnes-value at time~$t$ from the
target and $r(t,n)$ is an ``error term'' that we will prove to be $O(1/n)$ if $g(X_t)>\log n$. Moreover, $r(t,n)=0$ if $g(X_t)=0$.
The first step is to prove that $Y_t$ is a submartingale, \ie, $\E{Y_{t+1}\mid Y_t} \ge Y_t$. Afterwards, 
we bound the mgf. 
of $D_t=Y_{t}-Y_{t-1}=g(X_{t+1})-g(X_t)+\Prob(T>t)+r(t,n)$. 

\begin{lemma}
\label{lem:ytsubmartingale}
The sequence $Y_t=g(X_t)+\sum_{s=0}^{t-1}\Prob(T>s)+r(t,n)$ is a submartingale with $r(t,n)=O(1/n)$ for $X_t>\log n$.
\end{lemma}

\begin{proof}
We first note that nothing is to show if $g(X_t)=0$. Hence, we assume 
$g(X_t)>0$ in the following. From Section~\ref{sec:upper-tail-g} we know that $\E{g(X_t)-g(X_{t+1})\mid X_t} \ge 1$. 
Basically, this holds since the function~$g$ is concave and the slope at point $X_t$ equals $1/h(X_t)$, so 
that $g$ scales the drift of the $X_t$-process at this point, which is $h(X_t)$, by a factor of $1/h(X_t)$. 
We will next show that the error incurred by estimated the drift of $g$ via the slope $1/h(X_t)$ 
is small. The approach is similar to \cite{HwangWittFOGA19}, who analyzed this kind of 
error with respect to the drift function 
of the \ea on \OneMax.

We claim that if $X_t\ge \log n$ then
\[
\E{g(X_t)-g(X_{t+1})\mid X_t} \le 1+O(1/n).
\] 

To prove the claim, we show  
\[
\E{g(X_{t})-g(X_{t+1})\mid X_t} - \frac{\E{X_{t}-X_{t+1}\mid X_t}}{h(X_t)} = O(1/n)
\] 
instead. Substituting the definition of~$g$ and noting that $X_{t+1}\le X_t$, this is equivalent to
\[
\Ebig{\int_{X_{t+1}}^{X_t} 1/h(i)\,\mathrm{d}i \Bigm| X_t} - \frac{\E{X_{t}-X_{t+1}\mid X_t}}{h(X_t)}
\]
and, due to the discrete state space, it is also equivalent to
\[
\Ebig{\sum_{i=X_{t+1}}^{X_t} \frac{1}{h(i)}\Bigm| X_t} - \frac{\E{X_{t}-X_{t+1}\mid X_t}}{h(X_t)}.
\]
By the definition of expectation, this difference equals
\begin{align}
& \sum_{j=1}^{X_t}\Prob(X_{t+1}=X_t-j) \left(\sum_{k=X_t-j+1}^{X_{t}} \frac{1}{h(k)} - \frac{j}{h(X_t)}\right)\notag\\
& \quad \le \sum_{j=1}^{X_t}\Prob(X_{t+1}=X_t-j) \left(\frac{j}{h(X_t-j+1)}-\frac{j}{h(X_t)}\right),
\label{eq:upper-bound-g-difference-1}
\end{align}
where we used that $h$ is non-decreasing. After an index manipulation (formally, writing $j'=j-1$), 
we are left with the task of bounding
\[
\left(\frac{1}{h(X_t-j')}-\frac{1}{h(X_t)}\right)
\]
for $j'\ge 1$ (since for $j'=0$ the difference is~$0$). 
Writing the difference as 
\[
\frac{h(X_t)-h(X_t-j')}{h(X_t)h(X_t-j')}
\]
and using the bounds
\[
\frac{(1-1/n)^{n-i}}{2n} \le h(i) = \frac{(1-1/n)^{n-i}}{(2-(1/2)^{i-1})n} \le \frac{n(1+1/n)(1-1/n)^{n-i}}{2n}
\]
according to Lemma~\ref{lem:summary-lo} for $i\ge \log n$,
we obtain for $X_t\ge 2\log n$ and $j'\le \log n$ that
\begin{align*}
\frac{h(X_t)-h(X_t-j')}{h(X_t)h(X_t-j')} & \le 
\frac{4 n^2\left(1+\frac{1}{n}\right) \left(\left(1-\frac{1}{n}\right)^{n-X_t} - 
\left(1-\frac{1}{n}\right)^{n-X_t+j'}\right)}{2n\left(1-\frac{1}{n}\right)^{n-X_t+n-X_t+j'}}\\
& = \frac{2n\left(1+\frac{1}{n}\right) \left(1-\frac{1}{n}\right)^{n-X_t}\left(1-\left(1-\frac{1}{n}\right)^{j'}\right)}
{\left(1-\frac{1}{n}\right)^{n-X_t+n-X_t+j'}} \\
& \le 
2e \left(1+\frac{1}{n}\right)^2 j',
\end{align*}
where the last inequality used Bernoulli's inequality and the estimate $(1-1/n)^n \ge e^{-1}(1-1/n)$.

Together with the estimate $\Prob(X_{t+1}=X_t-j)\le \frac{1}{n}\left(\frac{1}{2}\right)^{j-1}$  from 
Lemma~\ref{lem:summary-lo} 
and recalling 
the index transformation, we bound 
\eqref{eq:upper-bound-g-difference-1} from above by 
\[
\sum_{j=1}^{X_t}  \frac{1}{n}\left(\frac{1}{2}\right)^{j-1} j 2e (1+1/n)^2 (j-1) = O(1/n)
\]
as suggested.
\end{proof} 

Recall that the aim is to apply 
Theorem~\ref{theo:martingale-diff-subexp-weak} with respect to the submartingale sequence 
$Y_t=g(X_t)+\sum_{s=0}^{t-1}\Prob(T>s)+r(t,n)$.  	
To this end, we shall bound the mgf.\  
of $D_t=Y_{t}-Y_{t-1}=g(X_{t+1})-g(X_t)+\Prob(T>t)+r(t,n)$ in the following way.
\begin{lemma}
\label{lem:mgf-lower-tail}
The mgf.\ of $D_t=Y_{t}-Y_{t-1}=g(X_{t+1})-g(X_t)+\Prob(T>t)+r(t,n)$ satisfies
$
\E{e^{\lambda D_t}\mid X_t} = e^{O(\lambda^2 n)} 
$	
for all $\lambda\in [1/n^2, 1/(2en)]$.
\end{lemma}

\begin{proof}
 Without loss of generality the process stops from time~$T$ on so that 
$g(X_{t+1})-g(X_t)=0$ for all $t\ge T$. Hence, 
conditional on $T\le t$ we have $\E{e^{\lambda D_t}} = e^{\lambda 0 + 0} = 1 \le e^{\lambda^2 \nu^2/2}$. We now consider the
interesting case that $T>t$. Then 
\begin{align*}
& \E{e^{\lambda D_t}\mid X_{t}} = \E{e^{\lambda(g(X_{t+1})-g(X_t)+1+r(t,n))}\mid X_t}
\end{align*} 
since $\Prob(T>t)=1$ on our condition. This mgf.\ differs from the one investigated in 
Lemma~\ref{lem:mgf-uppertail} and its proof only by the factor $e^{\lambda r(t,n)}$. Adjusting 
\eqref{eq:momgen3} accordingly and plugging in $r(n)=r(t,n)=O(1/n)$, we 
obtain
\begin{align}
& \E{e^{\lambda D_t}\mid X_t=i} \\\notag 
&  = 
(1-1/n)^{n-i} \frac{1}{n} (1+\lambda(1+r(n)))\\\notag
& \qquad +\lambda^2(1+r(n))^2) 
 \left(1 - \lambda n (1-1/n)^{i-n} + O(\lambda^2 n^2)\right)\\\notag
& \qquad + \left(1- (1-1/n)^{n-i}\frac{1}{n} \right) (1+\lambda(1+r(n))+\lambda^2(1+r(n))^2) \\\notag 
& = 1 + \lambda + \lambda^2 + \lambda r(n) + \lambda^2 (2r(n)+r^2(n)) - \lambda + O(\lambda ^2 n) \\\notag 
& \le e^{\lambda^2 + O(\lambda/n) + O(\lambda^2/n) + O(\lambda^2 n)} = e^{O(\lambda/n + \lambda^2 n)},  
\label{eq:momgen-lower3}
\end{align}
for $\lambda\le 1/(2en)$. Since  
$\lambda/n \le \lambda^2 n$ for $\lambda \ge 1/n^2$, we have   
$\E{e^{\lambda D_t}\mid X_t=i} = e^{O(\lambda^2 n)}$ for all $\lambda\in [1/n^2, 1/(2en)]$.
\end{proof}

Hence, we can 
satisfy the assumptions of Theorem~\ref{theo:martingale-diff-subexp-weak} with $b_2=2en$ and $b_1=n^2$. 
We will apply this theorem in the following subsection, where we put everything together.

\subsection{Main Concentration Result -- Putting Everything Together}
In the previous subsections we have derived (\wrt\ \LeadingOnes) that the sequence 
$\Delta^{(\ell)}_t=g(X_t)-g(X_{t+1})+\sum_{s=0}^{t-1}\Prob(T>s)$ is a supermartingale 
and the sequence 
$\Delta^{(h)}_t=g(X_t)-g(X_{t+1})+\sum_{s=0}^{t-1}\Prob(T>s)+r(t,n)$, where $r(t,n)=O(1/n)$, 
is a submartingale. We also know from Theorem~\ref{theo:variable-fixed-budget-expect} 
that 
$
\E{g(X_t)\mid \filtzero} \le g(X_0) - \sum_{s=0}^{T-1}\Prob(T>s).
$
Hence, using Theorem~\ref{theo:martingale-diff-subexp-weak} with respect to the $\Delta^{(\ell)}_t$-sequence, 
choosing $b_1=\infty$ and $b_2=2en$ according to our analysis of the mgf., we obtain (since $\nu^2=O(n)$) the first statement of the following theorem.
Its second statement follows by applying Theorem~\ref{theo:martingale-diff-subexp-weak} 
with respect to the $\Delta^{(h)}_t$-sequence, choosing $b_2=2en$ and $b_1=n^2$.

\begin{theorem}
\label{theo:tail-bounds-all-together}
\[
\Prob\bigl(g(X_t) \ge \E{g(X_t)} + d\bigr)
\le \begin{cases}
 e^{-d/(4en)}, & \text{ if $d\ge C t$;}\\
 e^{-\Omega(d^2/(tn))}, & \text{ otherwise},
\end{cases}
\]
where $C=\nu^2/(4en)=O(1)$.
Moreover, 
\[
\Prob\bigl(g(X_t) \le \E{g(X_t)} - d  - t r(t,n)\bigr)
\le \begin{cases}
 e^{-d/(4en)}, &  \text{ if $d\ge C t$;}\\
 e^{-\Omega(d^2/(tn))},\hspace{-2ex} & \text{ if $\frac{C't}{n} \le d<Ct$;}
\end{cases}
\]
where $C=\nu^2/(4en)=\Theta(1)$ and $C'=\nu^2/n=\Theta(1)$.
\end{theorem}

As mentioned above, Theorem~\ref{theo:variable-fixed-budget-expect} gives us an 
upper bound on $\E{g(X_t)}$ but we would like to 
know an upper bound on $\E{X_t}$. Unfortunately, 
since $g$ is concave, it does \textbf{not} hold that 
$
\E{X_t} \le g^{-1}(\E{g(X_t)}).
$
However, using the concentration inequalities above, we can show that $\E{X_t}$ is not much bigger than 
the right-hand side of this wrong estimate. Given $t>0$, we choose a $d^*>0$ for the tail bound such that 
$\Prob(g(X_t) > \E{g(X_t)} + d^*) \le 1/n^3$. If $g(X_t)\le \E{g(X_t)} + d^*$, the concavity of $g$ implies that the
 $\E{X_t}$-value is maximized 
if $g(X_t)$ takes the value $\E{g(X_t)} + d^*$  with probability $\frac{\E{g(X_t)}}{\E{g(X_t)} + d^*}$ and is $0$ 
otherwise. Since 
$g(X_t)=O(n^2)$, we altogether have
\begin{align*}
\E{X_t} & \le \frac{1}{n^3}O(n^2) + g^{-1}(\E{g(X_t)}+d^*) \frac{\E{g(X_t)}}{\E{g(X_t)} + d^*} \\ 
 & = g^{-1}(\E{g(X_t)}+d^*) \frac{\E{g(X_t)}}{\E{g(X_t)} + d^*} + o(1).
\end{align*}

We will now make this concrete and conclude by presenting the proof of our main theorem from this section.
\begin{proofof}{Theorem~\ref{theo:bounds-fixed-budget-final}}
First of all, we need some handy estimates for 
 $g(a)=\frac{1}{h(1)}+\sum_{1}^a \frac{1}{h(i)}\mathrm{d}i$. Since $i$ is an integer, we can 
also integrate over $1/h(\lceil i\rceil)$ instead and obtain
\[
g(a)=\sum_{i=1}^a \frac{1}{h(i)} = \sum_{i=1}^a (2-(1/2)^{i-1})^{-1} (1-1/n)^{i-n} n 
\] 
using the expression from Lemma~\ref{lem:summary-lo}.
So, 
\[
g(a) \le \frac{n}{2} \sum_{i=1}^a  (1-1/n)^{i-n} = \frac{n}{2} (n-1)(1-1/n)^{-n}  \left(1-\left(1-\frac{1}{n}\right)^{a}\right).
\]
Also, since $2-(1/2)^{i-1}\ge 2-2/n$ for $i\ge \log n$, we have
\begin{align}
\notag g(a) & \ge \left(1-\frac{1}{n}\right)\sum_{i=\log n+1}^a  \frac{1}{2}(1-1/n)^{i-n}  \\ 
& \notag\ge 
\frac{(n-1)^2}{2}(1-1/n)^{-n}  \left(1-\left(1-\frac{1}{n}\right)^{a}\right) - 2n\log n\\
& \ge 
\frac{en^2}{2}  \left(1-\left(1-\frac{1}{n}\right)^{a}\right) - 3n\log n,
\label{eq:g-lower-bound}
\end{align}
where the last bound holds for sufficiently large~$n$.

Let $X_t\coloneqq n-V_t$ be the fitness distance at time~$t$. 
Using Theorem~\ref{theo:tail-bounds-all-together} with $d=c^*\sqrt{tn\ln n}$, where $c^*$ is a sufficiently large constant, 
and noting that $d< Ct$ by our assumption on~$t$ (for $n$ large enough), we have 
\[
\Prob\bigl(g(X_t) \ge \E{g(X_t)} + d\bigr) \le 
 e^{-\Omega(d^2/(tn))} \le 1/n^3.
\]
So, if we have a bound on $\E{g(X_t)}$, we will obtain a bound on $\E{X_t}$ via the inverse of~$g$ 
computed above.

We  define $T^*=\frac{e-1}{2}n^2$ and note that this reflects the expected optimization time $\E{T}$ 
of the \ea on \LeadingOnes up to a relative error of $1+O(1/n)$ according to Lemma~\ref{lem:summary-lo}. 
 Since $t\le T^*-cn^{3/2}\sqrt{\log n}$, we obtain from Theorem~\ref{theo:djwz} by choosing $c$ sufficiently large that 
$\Prob(T\le t)\le 1/n^3$. 
Hence, from Theorem~\ref{theo:variable-fixed-budget-expect}, $\E{g(X_t)}\le \frac{e-1}{2}n^2 - t + O(1/n)$.

Altogether, 
\[
\Prob\biggl(g(X_t) \ge \frac{e-1}{2}n^2 -t + O(\sqrt{tn\log n})\biggr) 
  \le \frac{1}{n^3}.
	\]

We now invert $g$, more precisely the lower bound \eqref{eq:g-lower-bound} since $g$ is increasing and we want to bound 
the pre-image from above. Hence, we obtain that 
\[
z=g(a) \ge  
\frac{en^2}{2}  \left(1-\left(1-\frac{1}{n}\right)^{a}\right) - 3n\log n 
\]
implies
\[
a(z) \le \frac{\ln(1-\frac{2z+6n\log n}{en^2}) }{\ln(1-1/n)}.
\]
Using $z=\frac{e-1}{2}n^2 - t + O(\sqrt{tn\log n})$, we finally have
\begin{align*}
X_t & \le \frac{\ln(1-\frac{(e-1)n^2-2t+ O(\sqrt{tn\log n}) + 6n\log n}{en^2}) }{\ln(1-1/n)}  \\
 & = \frac{\ln(1/e-2t/(en^2)+O(\sqrt{t\log n}/n^{3/2}))}{\ln(1-1/n)} \\
& = \frac{-1+\ln(1-2t/n^2+O(\sqrt{t\log n}/n^{3/2}))}{\ln(1-1/n)}
\end{align*}
with probability at least $1-n^{-3}$. Estimating $\ln(1-1/n)=-1/n\pm O(1/n^2)$ and moving the error term into the logarithm, the lower bound 
on $V_t=n-X_t$ from the first claim follows of the theorem 
after 
straightforward manipulations.

The upper bound on $V_t$ 
is proved almost analogously. We use the $\Delta^{(h)}$ sequence instead of the $\Delta^{(\ell)}$ sequence and 
note that the $r(t,n)$ terms vanish in the $O(\sqrt{t\log n}/n^{3/2})$ errors. The only requirement that has to be met additionally is that 
Lemma~\ref{lem:ytsubmartingale} only holds for $X_t>\log n$. However, by a straightforward change 
of Theorem~\ref{theo:djwz} by choosing $c$ as a sufficiently large constant we have not only 
$\Prob(T\le t)\le 1/n^3$ but also $\Prob(\min\{s\mid X_s\le \log n\} \le t) \le 1/n^3$. 

For the expected value, we note that $X_t\le n$, so by the law of total probability, 
and again estimating $\ln(1-1/n)=-1/n\pm O(1/n^2)$, the claim follows. 
\end{proofof}

\section{{Conclusions}\label{sec:conclusions}}
We have described two general approaches that derive fixed-budget results via drift analysis. The first approach is concerned with iterating 
drifts either in an unbounded time scenario, or, using bounds on hitting times, in the scenario that the underlying process stops at some target state. 
Applying this approach to the \om or \lo functions, we obtain strong lower bounds on the expected fitness value after a given number of iterations.
The second approach is based on variable drift analysis and tail bounds for martingale differences. Exemplified for the \lo function, this technique 
allows us to derive statements that are more precise than the previous state of the art. We think that our drift theorems can be useful for 
future fixed-budget analyses.

\bibliographystyle{plainnat}

\begin{thebibliography}{18}
\providecommand{\natexlab}[1]{#1}
\providecommand{\url}[1]{\texttt{#1}}
\expandafter\ifx\csname urlstyle\endcsname\relax
  \providecommand{\doi}[1]{doi: #1}\else
  \providecommand{\doi}{doi: \begingroup \urlstyle{rm}\Url}\fi

\bibitem[B{\"o}ttcher et~al.(2010)B{\"o}ttcher, Doerr, and
  Neumann]{BDNLeadingOnes}
S{\"u}ntje B{\"o}ttcher, Benjamin Doerr, and Frank Neumann.
\newblock Optimal fixed and adaptive mutation rates for the {LeadingOnes}
  problem.
\newblock In \emph{Proc.\ of PPSN~2010}, volume 6238 of \emph{Lecture Notes in
  Computer Science}, pages 1--10. Springer, 2010.

\bibitem[Doerr et~al.(2013)Doerr, Jansen, Witt, and Zarges]{DJWZGECCO13}
Benjamin Doerr, Thomas Jansen, Carsten Witt, and Christine Zarges.
\newblock A method to derive fixed budget results from expected optimisation
  times.
\newblock In \emph{Proc.\ of GECCO~2013}, pages 1581--1588. ACM Press, 2013.

\bibitem[Doerr et~al.(2020)Doerr, Doerr, and Yang]{DoerrDYTCS20}
Benjamin Doerr, Carola Doerr, and Jing Yang.
\newblock Optimal parameter choices via precise black-box analysis.
\newblock \emph{Theoretical Computer Science}, 801:\penalty0 1--34, 2020.

\bibitem[Fan et~al.(2015)Fan, Grama, and Liu]{FanExponentialMartingales2015}
Xiequan Fan, Ion Grama, and Quansheng Liu.
\newblock Exponential inequalities for martingales with applications.
\newblock \emph{Electronic Journal of Probabability}, 20:\penalty0 22 pp.,
  2015.
\newblock URL \url{https://doi.org/10.1214/EJP.v20-3496}.

\bibitem[Hall et~al.(2019)Hall, Oliveto, and Sudholt]{HallOSGECCO19}
George~T. Hall, Pietro~Simone Oliveto, and Dirk Sudholt.
\newblock On the impact of the cutoff time on the performance of algorithm
  configurators.
\newblock In \emph{Proc.\ of GECCO~'19}, pages 907--915. ACM Press, 2019.

\bibitem[He et~al.(2019)He, Jansen, and Zarges]{HeJansenZargesArxiv20}
Jun He, Thomas Jansen, and Christine Zarges.
\newblock Unlimited budget analysis of randomised search heuristics.
\newblock \emph{CoRR}, abs/1909.03342, 2019.
\newblock URL \url{http://arxiv.org/abs/1909.03342}.

\bibitem[Hwang and Witt(2019)]{HwangWittFOGA19}
Hsien{-}Kuei Hwang and Carsten Witt.
\newblock Sharp bounds on the runtime of the {(1+1)} {EA} via drift analysis
  and analytic combinatorial tools.
\newblock In \emph{Proc.\ of FOGA~2015}, pages 1--12. {ACM} Press, 2019.
\newblock ISBN 978-1-4503-6254-2.

\bibitem[Jansen(2020)]{JansenFixedBudgetSurvery}
Thomas Jansen.
\newblock Analysing stochastic search heuristics operating a fixed budget.
\newblock In Benjamin Doerr and Frank Neumann, editors, \emph{Theory of
  Evolutionary Computation: Recent Developments in Discrete Search Spaces},
  pages 249--270. Springer, 2020.

\bibitem[Jansen and Zarges(2012)]{JansenZargesGECCO12}
Thomas Jansen and Christine Zarges.
\newblock Fixed budget computations: a different perspective on run time
  analysis.
\newblock In \emph{Proc.\ of GECCO~2012}, pages 1325--1332. ACM Press, 2012.

\bibitem[Jansen and Zarges(2014)]{JansenZargesTEC2014}
Thomas Jansen and Christine Zarges.
\newblock Reevaluating immune-inspired hypermutations using the fixed budget
  perspective.
\newblock \emph{{IEEE} Transactions on Evolutionary Computation}, 18\penalty0
  (5):\penalty0 674--688, 2014.

\bibitem[Johannsen(2010)]{Johannsen2010}
Daniel Johannsen.
\newblock \emph{Random Combinatorial Structures and Randomized Search
  Heuristics}.
\newblock PhD thesis, Universit{\"a}t des Saarlandes, Saarbr{\"u}cken, Germany
  and the Max-Planck-Institut f{\"u}r Informatik, 2010.

\bibitem[K{\"o}tzing(2016)]{KoetzingTailBoundAlgorithmica}
Timo K{\"o}tzing.
\newblock Concentration of first hitting times under additive drift.
\newblock \emph{Algorithmica}, 75:\penalty0 490–506, 2016.

\bibitem[Lehre and Witt(2014)]{LehreWittISAAC2014}
Per~Kristian Lehre and Carsten Witt.
\newblock Concentrated hitting times of randomized search heuristics with
  variable drift.
\newblock In \emph{Proc.\ of ISAAC~2014}, volume 8889 of \emph{Lecture Notes in
  Computer Science}, pages 686--697. Springer, 2014.
\newblock Extended version at \url{http://arxiv.org/abs/1307.2559}.

\bibitem[Lengler(2020)]{Lengler2020}
Johannes Lengler.
\newblock Drift analysis.
\newblock In Benjamin Doerr and Frank Neumann, editors, \emph{Theory of
  Evolutionary Computation: Recent Developments in Discrete Optimization},
  pages 89--131. Springer, 2020.

\bibitem[Lengler and Spooner(2015)]{LenglerSpoonerFOGA15}
Johannes Lengler and Nicholas Spooner.
\newblock Fixed budget performance of the {(1+1)} {EA} on linear functions.
\newblock In \emph{Proceedings of FOGA~2015}, pages 52--61. {ACM} Press, 2015.

\bibitem[McDiarmid(1998)]{McDiarmid1998}
Colin McDiarmid.
\newblock Concentration.
\newblock In M.~Habib, C.~McDiarmid, J.~Ramirez-Alfonsin, and B.~Reed, editors,
  \emph{Probabilistic Methods for Algorithmic Discrete Mathematics}, page
  195–247. Springer, 1998.

\bibitem[Nallaperuma et~al.(2017)Nallaperuma, Neumann, and
  Sudholt]{NallaperumaNSECJ17}
Samadhi Nallaperuma, Frank Neumann, and Dirk Sudholt.
\newblock Expected fitness gains of randomized search heuristics for the
  traveling salesperson problem.
\newblock \emph{Evolutionary Computation}, 25\penalty0 (4), 2017.

\bibitem[Wainwright(2015)]{WainwrightNotes2015}
M.~Wainwright.
\newblock Basic tail and concentration bounds.
\newblock Technical report, 2015.
\newblock Lecture Notes, Univ. of Berkeley,
  \url{https://www.stat.berkeley.edu/~mjwain/stat210b/Chap2_TailBounds_Jan22_2015.pdf}.

\end{thebibliography}

\end{document}